\documentclass[english,a4paper,12pt]{article}

\usepackage{amsxtra, amsfonts, amssymb, amstext, amsmath, mathtools}
\usepackage{amsthm}
\usepackage{booktabs}
\usepackage{nicefrac}
\usepackage{xspace}
\usepackage{url}
\usepackage{graphics,color}
\usepackage[algo2e,ruled,vlined,linesnumbered]{algorithm2e}
\usepackage{wrapfig}

\newtheorem{theorem}{Theorem}
\newtheorem{lemma}[theorem]{Lemma}
\newtheorem{corollary}[theorem]{Corollary}

\newtheorem{remark}[theorem]{Remark}

\newcommand{\algo}{\mbox{StSt~$\binom{\mu}{2}$~GA{$_0$}}\xspace}

\newcommand{\N}{\ensuremath{\mathbb{N}}} 

\DeclareMathOperator{\cross}{crossover}

\newcommand{\assign}{\leftarrow}

\begin{document}
\title{Runtime Analysis of Evolutionary Algorithms via Symmetry Arguments}

\author{Benjamin Doerr\\ Laboratoire d'Informatique (LIX)\\ CNRS\\ \'Ecole Polytechnique\\ Institut Polytechnique de Paris\\ Palaiseau\\ France
}

\maketitle

\sloppy{
\begin{abstract}
  We use an elementary argument building on group actions to prove that the selection-free steady state genetic algorithm analyzed by Sutton and Witt (GECCO 2019) takes an expected number of $\Omega(2^n / \sqrt n)$ iterations to find any particular target search point. This bound is valid for all population sizes $\mu$. Our result improves over the previous lower bound of $\Omega(\exp(n^{\delta/2}))$ valid for population sizes $\mu = O(n^{1/2 - \delta})$, $0 < \delta < 1/2$. 
\end{abstract}

\section{Introduction}\label{sec:eins}

The theory of evolutionary algorithms (EAs) has produced a decent number of mathematically proven runtime analyses. They explain the working principles of EAs, advise how to use these algorithms and how to choose their parameters, and have even led to the invention of new algorithms. We refer to \cite{AugerD11,DoerrN20,Jansen13,NeumannW10} for introductions to this area. 

Due to the complexity of the probability space describing a run of many EAs, the majority of the runtime analyses regard very simple algorithms. In particular, there are only relatively few works discussing algorithms that employ crossover, that is, the generation of offspring from two parents. Among these, again very few present lower bounds on runtimes; we are aware of such results only in~\cite{DoerrT09,OlivetoW15,SuttonW19}.

In the most recent of these works, Sutton and Witt~\cite[Section~3]{SuttonW19} consider a simple crossover-based algorithm called \algo (made precise in Section~\ref{sec:problem} below). This steady-state genetic algorithm uses a two-parent two-offspring uniform crossover as only variation operator. The two offspring always replace their parents. There is no fitness-based selection and no mutation in this simple process. Clearly, an algorithm of this kind is not expected to be very useful in practice. The reason to study such algorithms is rather that they allow to analyze in isolation how crossover works (more reasons to study this particular algorithm are described in~\cite{SuttonW19}). 

Without fitness-based selection, and thus without regarding the problem to be optimized, one would expect that this algorithm takes an exponential time to find any particular search point of the search space $\Omega = \{0,1\}^n$. Surprisingly, this is not so obvious, at least not when working with a particular initialization of the population. Sutton and Witt~\cite[Theorem~10]{SuttonW19} initialize the algorithm with $\mu/2$ copies of the string $z = (1010\dots10)$ and $\mu/2$ copies of the string $z' = (0101\dots01)$. They argue that this is a population with extremely high diversity, which could thus be beneficial for a crossover-based algorithm. Sutton and Witt show that their algorithm with this initialization and with population size $\mu = O(n^{1/2-\delta})$, $\delta < 1/2$ a constant, takes an expected number of $\Omega(\exp(n^{\delta/2}))$ iterations to generate the target string $x^* = (11\dots 1)$. Apparently, this lower bound is subexponential for all population sizes. It becomes weaker with increasing population size and is trivial for $\mu = \Omega(\sqrt n)$. 

By exploiting symmetries in the stochastic process, we improve the lower bound to $\Omega(2^n / \sqrt n)$ for all values of $\mu$.

\begin{theorem}\label{thm:main}
  Let $t, \mu, n \in \N$ with $\mu$ and $n$ even. Consider a run of the \algo initialized with $\mu/2$ copies of $z = (1010\dots10)$ and $\mu/2$ copies of $z' = (0101\dots01)$. Then the probability that the target string $x^* = (11\dots1)$ is generated in the first $t$ iterations is at most $2t / \binom{n}{n/2}$. In particular, the expected time to generate $x^*$ is at least $\tfrac 14 \binom{n}{n/2} = \Omega(2^n / \sqrt n)$.
\end{theorem}

Our proof is based on a simple group action or symmetry argument. We observe that the automorphisms of the hypercube $\{0,1\}^n$ (viewed as graph) commute with the operations of the \algo. Consequently, if an automorphism $\sigma$ stabilizes the initial individuals $z$ and $z'$ (that is, $\sigma(z) = z$ and $\sigma(z') = z'$), then for any $x \in \{0,1\}^n$ at all times $t$ the probability that the algorithm generates $x$ equals the probability that it generates $\sigma(x)$. 

From this symmetry, we conclude that if $B$ is the set of all $x$ such that there is an automorphism of the hypercube that stabilizes the initial individuals and such that $x = \sigma(x^*)$, then at all times the probability that $x^*$ is generated, is at most $1/|B|$. We compute that $B$ has exactly $\binom{n}{n/2}$ elements. Hence each search point generated by the \algo is equal to $x^*$ only with probability $\binom{n}{n/2}^{-1}$. A simple union bound over the $2t$ search points generated up to iteration $t$ gives the result. 

\section{Precise Problem Statement}\label{sec:problem}

The algorithm regarded in~\cite[Section~3]{SuttonW19}, called \algo, is a selection-free variant of a steady state genetic algorithm proposed earlier in~\cite{Witt18}. It works with the search space $\Omega = \{0,1\}^n$ of bit strings of length $n$, which is a standard representation used in evolutionary computation. The algorithm uses a population of size~$\mu \ge 2$. Each iteration consists of (i)~choosing two different individuals randomly from the population, (ii)~applying a two-offspring uniform crossover, and (iii)~replacing the two parents with the two offspring in the population. 

For two parents $x$ and~$y$, the two offspring $x'$ and $y'$ are generated as follows. For all $i \in [1..n] := \{1, \dots, n\}$ with $x_i = y_i$, we have $x'_i = y'_i = x_i$ with probability one. For all $i \in [1..n]$ with $x_i \neq y_i$, we have $(x'_i,y'_i) = (1,0)$ and $(x'_i,y'_i) = (0,1)$ each with probability $1/2$. 

The pseudocode of the \algo is given in Algorithm~\ref{alg:algo}. We did not specify a termination criterion since we are interested in how long the algorithm takes to find a particular solution when not stopped earlier. We also did not specify how to initialize the population since we will regard a very particular initialization later. We generally view populations as multisets, that is, an individual can be contained multiple times and this is reflected in the uniform random selection of individuals. Formally speaking, this means that a population is a $\mu$-tuple of individuals and individuals should be referred to via their index in this tuple.

\begin{algorithm2e}%
  $t \assign 0$\;
	Initialize $P_0$ with $\mu$ individuals from $\{0,1\}^n$\;
	\For{$t = 1, 2, \ldots$}{
	  Choose from $P_{t-1}$ two random individuals $x$ and $y$ without replacement\;
	  $(x_t,y_t) \assign \cross(x,y)$\;
	  $P_{t} \assign P_{t-1} \setminus \{x,y\} \cup \{x_t,y_t\}$\;     
    }
\caption{The \algo with population size $\mu \ge 2$ operating on the search space $\{0,1\}^n$.}
\label{alg:algo}
\end{algorithm2e}

Since a typical reason why crossover-based algorithms become inefficient is a low diversity in the population, Sutton and Witt consider an initialization of the \algo which has ``extremely high diversity'', namely $\mu/2$ copies of the string $z = (1010 \dots 10)$ and $\mu/2$ copies of the string $z' = (0101 \dots 01)$. This population has the same number of zeros and ones in each bit position and has the maximal number of pairs of individuals with maximal Hamming distance~$n$.\footnote{We recall that the \emph{Hamming distance} $H(x,y)$ of two bit strings $x, y \in \{0,1\}^n$ is defined by $H(x,y) = |\{i \in [1..n] \mid x_i \neq y_i\}|$.} Still, this initialization is fair with respect to the target of generating the string $x^* = (11 \dots 1)$ in the sense that all initial individuals have from $x^*$ a Hamming distance of $n/2$, which is the expected Hamming distance of a random string from $x^*$ (and the expected Hamming distance of any string from a random target).

\section{Proof of the Main Result}

We now prove our main result following the outline given towards the end of Section~\ref{sec:eins}. We do not assume any prior knowledge on groups and their action on sets. 

We view the hypercube $\{0,1\}^n$ as a graph in the canonical way, that is, two bit strings $x, y \in \{0,1\}^n$ are \emph{neighbors} if and only if they differ in exactly one position, that is, if $H(x,y) = 1$. A permutation $\sigma$ of $\{0,1\}^n$ is called \emph{graph automorphism} if it preserves the neighbor relation, that is, if $x$ and $y$ are neighbors if and only if $\sigma(x)$ and $\sigma(y)$ are neighbors. 

Let $G$ be the set of all graph automorphisms of the hypercube. We note that $G$ is a group. More precisely, $G$ is a subgroup of the symmetric group on $\{0,1\}^n$, that is, the group of all permutations of $\{0,1\}^n$ with the composition $\circ$ as group operation. Since different notations are in use, we fix that by $\circ$ we denote the usual composition of functions defined by $(\sigma_2 \circ \sigma_1)(x) = \sigma_2(\sigma_1(x))$ for all $\sigma_1, \sigma_2 \in G$ and $x \in \{0,1\}^n$. By regarding shortest paths, we easily observe that $G$ preserves the Hamming distance, that is, we have $H(x,y) = H(\sigma(x),\sigma(y))$ for all $\sigma \in G$ and $x,y \in \{0,1\}^n$. Hence $G$ is also the group of \emph{isometries} of the metric space $(\{0,1\}^n, H)$.

There are two types of natural automorphisms of the hypercube. 
\begin{itemize}
\item \emph{Rotations}: If $\pi$ is a permutation of $[1..n]$, then $\sigma_\pi$ defined by 
\[\sigma_\pi(x) = (x_{\pi(1)}, \dots, x_{\pi(n)})\]
for all $x \in \{0,1\}^n$ is the automorphism stemming from permuting the entries of $x$ as given by $\pi$. 
\item \emph{Reflections}: If $m \in \{0,1\}^n$, then $\sigma_m$ defined by 
\[\sigma_m(x) = x \oplus m\]
for all $x \in \{0,1\}^n$ is the automorphism stemming from adding the vector $m$ modulo two or, equivalently, performing an exclusive-or with~$m$.
\end{itemize}

We remark (without proof and without using this in our proofs) that the rotations and reflections generate $G$, and more specifically, that each $\sigma \in G$ can be written as product $\sigma = \sigma_\tau \circ \sigma_m$ for suitable $\tau$ and~$m$~\cite{Harary00}. 

The \emph{stabilizer} $G_x$ of a point $x \in \{0,1\}^n$ is the set of all permutations in $G$ fixing this point:
\[G_x := \{\sigma \in G \mid \sigma(x) = x\}.\]
To exploit symmetries in the stochastic process describing a run of the \algo, we now analyze the stabilizer of the initial individuals. Let $S = G_z$ be the stabilizer of the initial search point $z = (1010\dots10)$. We observe that $S$ also fixes the other initial individual $z' = (0101\dots01)$.

\begin{remark}\label{rem:zz}
  $S = G_{z'}$.
\end{remark}

\begin{proof}
Since $z'$ is the unique point with Hamming distance $n$ from $z$ and since $G$ is the group of isometries of the hypercube, any $\sigma \in G$ fixing $z$ also fixes $z'$, that is, $S$ is contained in the stabilizer $G_{z'}$ of $z'$. Via a symmetric argument, $G_{z'} \subseteq S$, and the claim follows.
\end{proof}

We proceed by determining a sufficiently rich subset of $S$. Obviously, any permutation $\tau$ of $[1..n]$ that does not map an even number to an odd one (and consequently does not map an odd number to an even one) has the property that $\sigma_\tau \in S$. Unfortunately, these automorphisms also fix $z^*$ and thus are not useful for our purposes. 

However, also the following automorphisms are contained in $S$. Let $i, j \in [1..n]$ such that $i$ is even and $j$ is odd. Let $m(i,j) \in \{0,1\}^n$ such that $m(i,j)_k = 1$ if and only if $k \in \{i,j\}$. Then $\sigma_{m(i,j)}$ changes zeros to ones and vice versa in the $i$-th and $j$-th position of $x$ and leaves all other positions unchanged. Let $\tau(i,j)$ be the permutation of $[1..n]$ that swaps $i$ and $j$ and fixes all other numbers. Then $\sigma(i,j) := \sigma_{\tau(i,j)} \circ \sigma_{m(i,j)}$ is contained in $S$. 

We use these automorphisms to give a lower bound on the size of the orbit $S(x^*)$ of $x^* = (1, \dots, 1)$ under~$S$. We recall that if $H$ is a group of permutations of a set $\Omega$ and $x \in \Omega$, then the \emph{orbit} $H(x) := \{\sigma(x) \mid \sigma \in H\}$ of $x$ under $H$ is the set of all elements to which $x$ can be mapped via a permutation of $H$. Since it might ease understanding this notion, we note that the orbits form a partition of $\Omega$, but we shall not build on this fact. Let $B := S(x^*)$ denote the orbit of $x^*$ under~$S$. We now determine $B$ and observe that it is relatively large.

\begin{lemma}\label{lem:orbit}
  $B$ consists of all $x \in \{0,1\}^n$ such that $H(x,z) = n/2$. In particular, $|B| = \binom{n}{n/2}$.
\end{lemma}

\begin{proof}
  We show first that $B$ cannot contain other elements. Let $x \in B$ and $\sigma \in S$ such that $x = \sigma(x^*)$. Then, using that $\sigma \in G_z$, $x = \sigma(x^*)$, and $\sigma$ is an isometry, we compute $H(x,z) = H(x,\sigma(z)) = H(\sigma(x^*),\sigma(z)) = H(x^*,z) = n/2$. 
  
  We now show that each $x \in \{0,1\}^n$ with $H(x,z) = n/2$ is contained in~$B$. Let $D = \{i \in [1..n] \mid x_i \neq x^*_i\}$. Since $H(x,z) = H(x^*,z)$, we have $x_i = z_i \neq x^*_i$ for exactly half of the $i \in D$ and $x_i \neq z_i = x^*_i$ for the other half. Consequently, $|D|$ is even and there are $k = |D|/2$ distinct even numbers $i_1, \dots, i_k \in [1..n]$ and $k$ distinct odd numbers $j_1, \dots, j_k \in [1..n]$ such that $x_{i_\ell} = z_{i_\ell}$ and $x_{j_\ell} \neq z_{j_\ell}$ for all $\ell \in [1..k]$ and these $2k$ positions are exactly the positions $x$ and $x^*$ differ in. Consequently, $\sigma = \sigma(i_1,j_1) \circ \dots \circ \sigma(i_k,j_k)$ is in~$S$ and satisfies $\sigma(x^*) = x$. 
\end{proof}

We finally argue that the actions of the rotations and reflections in $G$ are compatible with the operations of the \algo. Naturally, this implies the same compatibility statement for automorphisms that can be written as product of rotations and reflections. While we do not need this, we remark that via the above-mentioned result that $G$ is generated by rotations and reflections, the compatibility extends to the full automorphism group $G$. 

We lift the notation of an automorphism to random search points in the obvious way. If $X$ is a random search point (formally speaking, a random variable taking values in $\{0,1\}^n$), then $\sigma(X)$ is the random variable $\sigma \circ X$ defined on the same probability space. Consequently, and equivalently, we have $\Pr[\sigma(X) = x] = \Pr[X = \sigma^{-1}(x)]$ for all $x \in \{0,1\}^n$. Finally, we lift function evaluations to tuples in the obvious way so that, e.g., for the case of pairs we have $\sigma(x,y) = (\sigma(x),\sigma(y))$ for all $x,y \in \{0,1\}^n$ and all functions $\sigma$ defined on $\{0,1\}^n$.

We start by analyzing the crossover operation. We recall that the crossover operator used by the \algo is a randomized operator that generates a pair of search points from a given pair of search points. For all $x, y \in \{0,1\}^n$, the result $(X,Y) := \cross(x,y)$ of applying crossover to $(x,y)$ is distributed as follows. For $a,b,c,d \in \{0,1\}$ define $p(a,b,c,d)$ by
\[
p(a,b,c,d) = 
\begin{cases}
1 & \mbox{if $a=b=c=d$,}\\
\tfrac 12 & \mbox{if $a \neq b$ and $c \neq d$,}\\
0 & \mbox{otherwise}.
\end{cases}
\]
Then for all $u,v \in \{0,1\}^n$, we have 
\[
\Pr[\cross(x,y) = (u,v)] = \prod_{i=1}^n p(x_i,y_i,u_i,v_i).
\]

We use this description of the crossover operation to show that crossover commutes with rotations and reflections.

\begin{lemma}\label{lem:crossover}
  Let $\sigma \in G$ be a rotation or a reflection. Then for all $x,y \in \{0,1\}^n$, $\sigma(\cross(x,y))$ and $\cross(\sigma(x),\sigma(y))$ are equally distributed. This statement remains true if $x$ and $y$ are random search point with randomness stochastically independent from the one used by the crossover operator. These statements remain true when $\sigma$ is a product of a finite number of rotations and reflections.
\end{lemma}
	
\begin{proof}
  Let first $\sigma = \sigma_\tau$ be a rotation induced by some permutation $\tau$ of $[1..n]$. Then for all $u,v \in \{0,1\}^n$, we compute
  \begin{align*}
  \Pr[\cross(\sigma(x),\sigma(y)) = (u,v)] &= \prod_{i=1}^n p(\sigma(x)_i,\sigma(y)_i,u_i,v_i)\\
  &=\prod_{i=1}^n p(x_{\tau(i)},y_{\tau(i)},u_i,v_i),\\
  \Pr[\sigma(\cross(x,y)) = (u,v)] & = \Pr[\cross(x,y) = (\sigma^{-1}(u),\sigma^{-1}(v))] \\
  &=\prod_{i=1}^n p(x_i,y_i,u_{\tau^{-1}(i)},v_{\tau^{-1}(i)})\\
  &=\prod_{i=1}^n p(x_{\tau(i)},y_{\tau(i)},u_i,v_i),
  \end{align*}
  showing the desired equality of distributions. 
	
  Let now $\sigma = \sigma_m$ be a reflection induced by some $m \in \{0,1\}^n$. By definition, $p(a,b,c,d) = p(a \oplus r,b \oplus r, c \oplus r, d \oplus r)$ for all $a, b, c, d, r \in \{0,1\}$. Noting that $\sigma = \sigma^{-1}$, we compute
  \begin{align*}
  \Pr[\cross(\sigma(x),\sigma(y)) = (u,v)] &= \prod_{i=1}^n p(\sigma(x)_i,\sigma(y)_i,u_i,v_i)\\
  &=\prod_{i=1}^n p(x_i \oplus m_i,y_i \oplus m_i,u_i,v_i),\\
  \Pr[\sigma(\cross(x,y)) = (u,v)] & = \Pr[\cross(x,y) = (\sigma^{-1}(u),\sigma^{-1}(v))] \\
  &=\prod_{i=1}^n p(x_i,y_i,u_i \oplus m_i,v_i \oplus m_i)\\
  &=\prod_{i=1}^n p(x_i \oplus m_i,y_i \oplus m_i,u_i,v_i),\\
  \end{align*}
  showing again the desired equality of distributions. This shows the first result in Lemma~\ref{lem:crossover}.
  
  Assume now that $X$ and $Y$ are random search points, that is, that $(X,Y)$ is a random variable defined on some underlying probability space that is independent from the randomness used by the crossover operator. Let $\sigma$ be a rotation or reflection. From the first part of Lemma~\ref{lem:crossover}, we deduce 
  \begin{align*}
  \Pr&[\cross(\sigma(X),\sigma(Y)) = (u,v)] \\
  &= \sum_{x, y \in \{0,1\}^n} \Pr[(X,Y) = (x,y)] \Pr[\cross(\sigma(x),\sigma(y)) = (u,v)]\\
  &= \sum_{x, y \in \{0,1\}^n} \Pr[(X,Y) = (x,y)] \Pr[\sigma(\cross(x,y)) = (u,v)]\\
  &= \Pr[\sigma(\cross(X,Y)) = (u,v)]
  \end{align*}
  for all $u,v \in \{0,1\}^n$. This shows the second claim. 
  
  An elementary induction extends our claims to products of rotations and reflections.
\end{proof}

With Lemma~\ref{lem:crossover}, we now show that rotations and reflections commute (in a suitable sense) with the whole run of the \algo. 

\begin{lemma}\label{lem:invariance}
  Let $\sigma$ be a product of rotations and reflections in $G$. Consider a run of the \algo with initial population $P_0$. Let $X_t$ and $Y_t$ be the random variables describing the two search points generated in iteration $t$. Consider also an independent run of the \algo with initial population $P_0' = \sigma(P_0)$. Let $X'_t$ and $Y'_t$ denote the search points generated in iteration $t$ of this run. Then $(X'_t,Y'_t)$ and $\sigma(X_t,Y_t)$ are identically distributed.
\end{lemma}

\begin{proof}
  Consider the two runs with independent randomness. Denote by $P_t$ and $P'_t$ the populations generated in iteration $t$. For the sake of precision, we now take the view that a population is a $\mu$-tuple of individuals. We show that if $\sigma(P_{t-1})$ and $P'_{t-1}$ are identically distributed, then so are $\sigma(X_t,Y_t,P_t)$ and $(X'_t,Y'_t,P'_t)$. Since $\sigma(P_0)$ and $P'_0$ are identically distributed by assumption, the claim follows by induction over time.
  
  Assume now that for some $t$, the populations $\sigma(P_{t-1})$ and $P'_{t-1}$ are identically distributed. We consider a particular outcome of $P_{t-1}$ that occurs with positive probability. Let $\sigma(P_{t-1})$ be the corresponding outcome of $P'_{t-1}$. Since both these outcomes have the same probability of appearing, we can condition on both and show that $\sigma(X_t,Y_t,P_t)$ and $(X'_t,Y'_t,P'_t)$ are identically distributed in this conditional probability space. 
	
	Let $i, j \in [1..\mu]$ be different. The probability that these are the indices of the two individuals chosen as crossover parents is $\binom{\mu}{2}^{-1}$ in both runs of the \algo. So again we condition on this same outcome in both runs. Now the parents $(x,y)$ in the first run and the parents $(x',y')$ in the second run satisfy $\sigma(x,y) = (x',y')$. By Lemma~\ref{lem:crossover}, $\sigma(\cross(x,y))$ and $\cross(x',y')$ are equally distributed. Since these pairs of search points replace $(x,y)$ and $(x',y')$ in the respective populations, we see that $\sigma(P_t)$ and $P'_t$ are identically distributed when conditioning on $P_{t-1}, P'_{t-1}, i, j$, and, as discussed above, also without conditioning. This concludes the proof. 
\end{proof}

With Lemma~\ref{lem:invariance}, we can easily argue that at each time $t$ all elements of the orbits under $S$ have the same chance of being generated. We need and formulate this statement only for the orbit $B = S(x^*)$.

\begin{corollary}\label{cor:equi}
  Consider a run of the \algo started with the particular initialization $P_0$ described in Section~\ref{sec:problem}. Let $x \in B$ and $t \in \N$. Then $\Pr[x_t = x] = \Pr[x_t = x^*]$ and $\Pr[y_t = x] = \Pr[y_t = x^*]$.
\end{corollary}

\begin{proof}
  Let $\sigma \in S$ such that $\sigma(x^*) = x$. Besides the run of the \algo started with $P_0$, consider a second independent run started with $\sigma(P_0)$. Denote the offspring generated in this run by $x'_t$ and $y'_t$. By Lemma~\ref{lem:invariance}, 
  \[\Pr[x_t = x^*] = \Pr[\sigma(x_t) = \sigma(x^*)] = \Pr[x'_t = \sigma(x^*)] = \Pr[x'_t = x].\] 
  
  Since $\sigma$ is from $S$, by Remark~\ref{rem:zz}, we also have $\sigma(z') = z'$. Hence $P_0 = \sigma(P_0)$. Having the same initial population, the two runs regarded are identically distributed. Consequently, $x_t$ and $x'_t$ are identically distributed and we have 
  \[\Pr[x_t = x^*] = \Pr[x'_t = x] = \Pr[x_t = x].\]   
\end{proof}

We are now in the position to give a proof of Theorem~\ref{thm:main} stated in the introduction.

\begin{proof}
Since $\Pr[x_t = x^*] = \Pr[x_t = x]$ for all $x \in B$ and $t \in \N$ by Corollary~\ref{cor:equi}, we have
\begin{equation}\label{eq:main}
  \Pr[x_t = x^*] = \frac{1}{|B|} \sum_{x \in B} \Pr[x_t = x] = \frac{1}{|B|} \Pr[x_t \in B] \le \frac{1}{|B|}.
\end{equation}  
Naturally, the same estimate holds for $\Pr[y_t = x^*]$. Let the random variable $T$ denote the first iteration in which the search point $x^*$ is generated. Then a simple union bound over time and over the two offspring generated per iteration gives
\[\Pr[T \le t] \le \sum_{i=1}^t (\Pr[x_i = x^*] + \Pr[y_i = x^*]) \le \frac{2t}{|B|} = 2t / \binom{n}{n/2},\]
where the last equality follows from Lemma~\ref{lem:orbit}.

For the bound on the expectation of $T$, we use the standard argument $E[X] = \sum_{x = 1}^\infty \Pr[X \ge x]$ valid for all random variables $X$ distributed on the non-negative integers and compute
\begin{align*}
  E[T] &= \sum_{t = 1}^\infty \Pr[T \ge t]\\
  &\ge \sum_{t = 1}^\infty \max\left\{0, 1 - 2(t-1) \Big/ \binom{n}{n/2}\right\}\\
  & = \sum_{i = 1}^{\frac 12 \binom{n}{n/2}} 2i \Big/\binom{n}{n/2} \ge \frac 14 \binom{n}{n/2}.
\end{align*}
\end{proof}

\section{Conclusion and Open Problems}

We proposed an alternative approach to the problem how long the \algo with a particular initialization takes to generate a particular search point~\cite[Section~3.1]{SuttonW19}. Our lower bound of order $\Omega(2^n / \sqrt n)$, valid for all population sizes $\mu$, is significantly stronger than the previous result, which is at most $\Omega(\exp(n^{1/4}))$ and decreases with increasing population size until it is trivial for $\mu = \Omega(\sqrt n)$. Our main argument based on group actions is elementary and natural, which gives us the hope that similar arguments will find applications in other analyses of EAs. 

We believe that our lower bound is close to the truth, which we expect to be $\Theta(2^n)$, but we do not have a proof for this conjecture (in fact, we do not even know if the runtime is exponential -- unfortunately, the few existing exponential upper bounds only regard mutation-based EAs, see~\cite{Doerr20ppsnUB}). 

We note that when using a random initialization instead of the particular one proposed in~\cite{SuttonW19}, then a lower bound of $\Omega(2^n)$ follows simply from the fact that each search point that is generated is uniformly distributed. This argument, in a sense a toy version of ours, is apparently not widely known in the community; it was used in~\cite[Theorem~1.5.3]{Doerr20bookchapter} for a problem which previously~\cite[Theorem~5]{OlivetoW11} was attacked with much deeper methods. 

It thus seems that the difficulty of the problem posed in~\cite{SuttonW19} not only stems from the use of crossover, but also from the fact that a non-random initialization was used. We note that so far the impact of different initializations has not been discussed intensively in the literature on runtime analysis of EAs. The only works we are aware of are~\cite{Sudholt13,LaillevaultDD15,DoerrD16}. 

In the light of this state of the art, the two \textbf{open problems} of improving the lower bound to $\Omega(2^n)$ and showing an exponential upper bound appear interesting. Any progress here might give us a broader understanding how to analyze EAs using crossover or non-random initializations.

\subsection*{Acknowledgment}

This work was supported by a public grant as part of the
Investissement d'avenir project, reference ANR-11-LABX-0056-LMH,
LabEx LMH.

}


\end{document}